\documentclass[11pt]{article}
\usepackage{paper}

\title{New Insights into Learning with Correntropy Based Regression}
\begin{document}

    \date{}

    \author[]{Yunlong Feng \thanks{ylfeng@albany.edu}}
    \affil[]{Department of Mathematics and Statistics, University at Albany}

    \maketitle

\begin{abstract}
\noindent Stemming from information-theoretic learning, the correntropy criterion and its applications to machine learning tasks have been extensively studied and explored. Its application to regression problems leads to the robustness enhanced regression paradigm -- namely, correntropy based regression. Having drawn a great variety of successful real-world applications, its theoretical properties have also been investigated recently in a series of studies from a statistical learning viewpoint.  The resulting big picture is that correntropy based regression regresses towards the conditional mode function or the conditional mean function robustly under certain conditions. Continuing this trend and going further, in the present study, we report some new insights into this problem. First, we show that under the additive noise regression model, such a regression paradigm can be deduced from minimum distance estimation, implying that the resulting estimator is essentially a minimum distance estimator and thus possesses robustness properties. Second, we show that the regression paradigm, in fact, provides a unified approach to regression problems in that it approaches the conditional mean, the conditional mode, as well as the conditional median functions under certain conditions. Third, we present some new results when it is utilized to learn the conditional mean function by developing its error bounds and exponential convergence rates under conditional $(1+\epsilon)$-moment assumptions. The saturation effect on the established convergence rates, which was observed under $(1+\epsilon)$-moment assumptions, still occurs, indicating the inherent bias of the regression estimator. These novel insights deepen our understanding of correntropy based regression, help cement the theoretic correntropy framework, and also enable us to investigate learning schemes induced by general bounded nonconvex loss functions. 
\end{abstract}

\section{Introduction and Preliminaries}
In this paper, we are concerned with the regression problem, which aims at learning a regression function between input and output from given observations drawn from some unknown distribution. Such a regression function could typically be the conditional mean function, the conditional median function, or the conditional mode function, depending on the needs. To mathematically describe a regression procedure, let us denote $X$ as the input variable that takes value in a compact subset $\mathcal{X}\subset \mathbb{R}^d$ and $Y$ the continuous output variable taking values in $\mathbb{R}$. Assume that the given observations $\mathbf{z}=\{(x_i,y_i)\}_{i=1}^n$ are drawn independently from a certain unknown probability distribution $\rho$ over $\mathcal{X}\times\mathcal{Y}$ with $\rho_\mathsmaller{X}$ being its marginal distribution and $\rho_\mathsmaller{Y|X}$ the conditional distribution conditioned on $X$. For any fixed realization of $X$, the goal of regression is to learn a location parameter of the conditional distribution $\rho_\mathsmaller{Y|X}$. Recall that mean, median, and mode are three canonical location parameters of a probability distribution. And as such, typically, a regression paradigm regresses towards the conditional mean function $\mathbb{E}(Y|X)$, conditional median function ${\sf{median}}(Y|X)$, or the conditional mode function ${\sf{mode}}(Y|X)$. The resulting regression procedure is termed as mean regression, median regression, or modal regression, respectively. In this study, we terminologically term the three functions as \textit{location functions}. Additionally, we term the conditional quantile function and the conditional expectile function as \textit{generalized location functions}. Under these terminologies, learning for regression essentially learns a (generalized) location function. To further our discussion, throughout this study we consider the following general additive noise regression model 
\begin{align}\label{data_generating}
Y=f^\star(X)+\varepsilon,
\end{align}
where $\varepsilon$ is the noise centered around $0$. It is obvious that by assuming that $\mathbb{E}(\varepsilon|X)=0$, ${\sf{median}}(\varepsilon|X)=0$, or ${\sf{mode}}(\varepsilon|X)=0$, the underlying truth function $f^\star$ is essentially a location function of the conditional distribution $\rho_\mathsmaller{Y|X}$ and so our purpose of regression is to learn such a location function. It should be remarked that the above three location assumptions on the conditional noise distribution are, in fact, mild ones as otherwise, one can always translate the distributions of $\varepsilon|X$ to fulfill one of these assumptions. 

In the statistics and machine learning literature, one of the most frequently employed approaches to learning $f^\star$ is the empirical risk minimization induced by the least squares loss, which leads to the least squares regression and can be deduced from maximum likelihood estimation under the Gaussian noise assumption. However, in the presence of misspecification of the likelihood function, learning $f^\star$ through least squares regression may not work well due to the use of the least squares loss that amplifies large residuals. To address this problem, tremendous approaches have been proposed in the literature, a representative one of which is M-estimation \cite{huber2009robust,hampel2011robust,maronna2006robust}. The idea is to consider maximum likelihood estimation of the location parameter of a distribution based on longer-tailed distributional assumptions. Carrying over the idea to regression problems, one arrives at various regression M-estimators. In the literature, many other efforts have also been made to address this problem beyond the maximum likelihood framework. 

In this study, we will investigate an alternative approach that stems from information-theoretic learning, namely, Maximum Correntropy Criterion based Regression (MCCR) \cite{liu2007correntropy,principe2010information,chen2012maximum}, and has been finding a wealth of applications in machine learning and data science \cite{he2011maximum,zhang2013robust,chen2016generalized,chen2017maximum}.

\subsection{MCCR: An Information-Theoretic Learning Approach to Regression}
With the $n$ i.i.d observations $\mathbf{z}=\{(x_i,y_i)\}_{i=1}^n$, MCCR can be formulated as
\begin{align}\label{MCCR}
f_{\mathbf{z},\sigma}=\arg\max_{f\in\mathcal{H}}\frac{1}{n}\sum_{i=1}^n \exp\left(-\frac{(y_i-f(x_i))^2}{\sigma^2}\right),
\end{align}
where $\mathcal{H}$ is a hypothesis space chosen as a compact subset of $C(\mathcal{X})$ in this study and $\sigma>0$ a scale parameter. The motivation of introducing MCCR comes from the minimization of the Renyi's quadratic entropy of the residual, i.e., $-\log \mathbb{E}p_e(e)$ where $p_e$ is the density function of the residual variable $e$ in regression. Notice that minimizing $-\log\mathbb{E}p_e(e)$ can be equivalently cast as the maximization of $\mathbb{E}p_e(e)$. Assuming a Gaussian prior on the residual and considering the empirical counterpart of $\mathbb{E}p_e(e)$, one then has the MCCR formulation in \eqref{MCCR}. Such an entropy minimization interpretation of \eqref{MCCR} illustrates the terminology -- \textit{correntropy} \cite{liu2007correntropy,principe2010information}. 
 
It is obvious that MCCR can be also reformulated using the language of empirical risk minimization (ERM), as done recently in \cite{fenglearning} where, by introducing the loss function $\ell_\sigma(t)=\sigma^2(1-e^{-t^2/\sigma^2})$, the following ERM scheme is studied
\begin{align}\label{ERM_MCCR}
f_{\mathbf{z},\sigma}=\arg\min_{f\in\mathcal{H}}\frac{1}{n}\sum_{i=1}^n \ell_\sigma(y_i-f(x_i)).  
\end{align}
It is due to this reasoning that $f_{\mathbf{z},\sigma}$ is traditionally viewed as an M-estimator. Recently, its theoretical properties have been explored continuously in a series of studies. For instance, inspired by the work in \cite{hu2013learning,fan2014consistency}, \cite{fenglearning} demonstrated that MCCR can deal with mean regression robustly in the sense that only a fourth-moment condition on the response variable is needed to guarantee its convergence. Such a moment condition is further relaxed to the $(1+\epsilon)$-th order moment condition in \cite{feng2020learning}, encompassing  the case when the noise possesses infinite variance; in \cite{feng2017statistical}, it is shown that MCCR performs modal regression under certain restrictions to the noise; using Huber's contamination model for modeling outliers, \cite{feng2018learning} makes some efforts in order to explain the outlier-robustness of MCCR and shows that it can be utilized to learn $f^\star$ in the presence of outliers.  Learning theory assessments from algorithmic viewpoints are conducted in \cite{guo2018gradient,hu2020kernel} and assessments from an optimization viewpoint are conducted in \cite{syed2014optimization}. In addition, there are also some existing studies in the literature investigating the penalized version of the ERM scheme \eqref{ERM_MCCR} under bounded noise assumptions; see e.g., \cite{chen2018kernel,lv2019optimal,li2020error}. Here, we will take a step further towards the understanding of the correntropy based regression scheme 
\eqref{ERM_MCCR} by providing new insights into it.

\subsection{New Insights Brought by This Study}
In this study, the following new insights will be brought to attention: first, under the additive noise regression model \eqref{data_generating}, we show that $f_{\mathbf{z},\sigma}$ can be retrieved from minimum distance estimation. Here, the distance refers to the squared distance between the two densities $p_\mathsmaller{\varepsilon|X}$ and $p_\mathsmaller{E_f|X}$ integrated over $\mathcal{X}$ where $p_\mathsmaller{\varepsilon|X}$ is the conditional density of the noise $\varepsilon$ and $p_\mathsmaller{E_f|X}$ the conditional density of the residual variable $Y-f(X)$ for any $f:\mathcal{X}\rightarrow\mathbb{R}$. To put it simply, $f_{\mathbf{z},\sigma}$ is essentially a minimum distance estimator and so may outperform other regression estimators in terms of robustness; second, it is shown that MCCR provides a unified approach to learning location functions in that under different location assumptions on the distributions of $\varepsilon|X$, $f^\star$ may represent different location functions. The adaptiveness of MCCR allows us to tune the scale parameter $\sigma$ to adjust the regression function to which the regression scheme targets; third, when mean regression is of interest, we show that improved exponential type convergence rates of $f_{\mathbf{z},\sigma}$ can be established under the conditional $(1+\epsilon)$-moment assumption. Moreover, the saturation effect, which is observed in an existing study in \cite{feng2020learning} under a relaxed moment assumption, still occurs under the conditional $(1+\epsilon)$-moment assumption. More detailed speaking, there exists a threshold value of $\epsilon$ above which the convergence rates may be independent of $\epsilon$. As a result, imposing stronger moment conditions may not help improve the established convergence rates of the estimator, implying the existence of an inherent bias in mean regression. These novel insights can help cement the theoretical correntropy framework developed recently in \cite{fenglearning,feng2017statistical,feng2018learning,feng2020learning}.      

The rest of this paper is organized as follows. In Section \ref{sec::interpretation}, we report a minimum distance estimation interpretation of MCCR. Section \ref{sec::Unified} illustrates the unified approach that MCCR provides in learning location functions. Specifically, Section \ref{subsec::mean} is devoted to the investigation of mean regression under weak moment conditions. Section \ref{subsec::modal} is concerned with regression towards the conditional mode function. Section \ref{subsec::median} discusses the case of learning the conditional median function under noise restrictions. Numerical validations are provided in Section \ref{sec::experiments}. We conclude the paper in Section \ref{sec::conclusion} and summarize the related future studies.

\section{MCCR: A Minimum Distance Estimation Interpretation}\label{sec::interpretation}

Under the additive noise data-generating model \eqref{data_generating}, for any measurable function $f:\mathcal{X}\rightarrow \mathbb{R}$, we denote $E_f$ as the random variable defined by the residual between $Y$ and $f(X)$, i.e., $E_f=Y-f(X)$. Then, for any fixed realization of $X$, say $x$, the density function of $E_f|X=x$ can be obtained by translating that of $\varepsilon|X=x$ horizontally $f^\star(x)-f(x)$ units. Consequently, the density of $E_f|X=x$ can be expressed as
\begin{align*}
p_\mathsmaller{E_f|X=x}(t)=p_\mathsmaller{\varepsilon|X=x}(t+f(x)-f^\star(x)).
\end{align*}
Similarly, we also have
\begin{align*}
p_\mathsmaller{\varepsilon|X=x}(t)=p_\mathsmaller{E_f|X=x}(t+f^\star(x)-f(x)).
\end{align*}
Moreover, as realized in \cite{fan2014consistency}, 
\begin{align*}
p_\mathsmaller{E_f}(t)=\int_{\mathcal{X}}p_\mathsmaller{\varepsilon|X=x}(t+f(x)-f^\star(x))\mathrm{d}\rho_X(x)
\end{align*}
defines a density function of the random variable $E_f$, and 
\begin{align*}
p_\mathsmaller{\varepsilon}(t)=\int_{\mathcal{X}}p_\mathsmaller{E_f|X=x}(t+f^\star(x)-f(x))\mathrm{d}\rho_X(x)
\end{align*}
defines a density function of the random variable $\varepsilon$.  In what follows, we consider only the case where $p_\mathsmaller{\varepsilon|X}$ is uniformly bounded by a constant that is independent of $X$. 

In learning for regression problems, we are concerned with the estimation of the unknown truth function $f^\star$. If a function $f$ is exactly the same as the function $f^\star$ on $\mathcal{X}$, then according to the above statements, $p_\mathsmaller{E_f|X}$ would be exactly the same as $p_\mathsmaller{\varepsilon|X}$ pointwisely, namely, the translation between the two densities would be zero for any fixed $x$. If for any fixed $x$, $f$ is, pointwisely, a good estimate of $f^\star$, then $p_\mathsmaller{E_f|X}$ may also mimic $p_\mathsmaller{\varepsilon|X}$ well. In other words, $p_\mathsmaller{E_f|X}$ may not departure too much from $p_\mathsmaller{\varepsilon|X}$. To measure such a deviation between the two distributions, we define the following integrated squared density-based distance.  

\begin{definition}[Integrated Squared Density-based Distance]\label{definition_distance}
	Let $\mathcal{M}$ be the function set that consists of all bounded measurable functions $f:\mathcal{X}\rightarrow [-M,M]$ with $M>0$ a constant. For any $f\in\mathcal{M}$, the integrated squared density-based distance between $p_\mathsmaller{E_f}$ and $p_\mathsmaller{\varepsilon}$, ${\sf{dist}}(p_\mathsmaller{E_f},p_\mathsmaller{\varepsilon})$, is defined as
	\begin{align*}
	{\sf{dist}}(p_\mathsmaller{E_f},p_\mathsmaller{\varepsilon}):=\int_\mathcal{X}\int_{-\infty}^{+\infty}(p_\mathsmaller{E_f|X=x}(t)-p_\mathsmaller{\varepsilon|X=x}(t))^2\mathrm{d}t\,\mathrm{d}\rho_X(x).
	\end{align*}
	\end{definition}
Throughout this paper, we assume that the truth function $f^\star$ is bounded by $M$, i.e., $\|f^\star\|_\infty\leq M$. It is easy to see that the square root of ${\sf{dist}}(p_\mathsmaller{E_f},p_\mathsmaller{\varepsilon})$ defines a metric between $p_\mathsmaller{E_f}$ and $p_\mathsmaller{\varepsilon}$.  In particular, if $f$ equals $f^\star$ on $\mathcal{X}$, then we have ${\sf{dist}}(p_\mathsmaller{E_f},p_\mathsmaller{\varepsilon})=0$. These observations, together with Definition \ref{definition_distance}, remind us that, if one would like to find a good estimate of the unknown location function $f^\star$ within a hypothesis space $\mathcal{H}$, then a possible strategy is to look for a function in $\mathcal{H}$, say $f_\mathcal{H}$, such that 
\begin{align}\label{density_based_strategy}
f_\mathcal{H}=\arg\min_{f\in\mathcal{H}}{\sf{dist}}(p_\mathsmaller{E_f},p_\mathsmaller{\varepsilon}).
\end{align}  
That is, one may seek the minimizer of the functional ${\sf{dist}}(p_\mathsmaller{E_f},p_\mathsmaller{\varepsilon})$ with respect to $f$ over $\mathcal{H}$ and use it to approximate $f^\star$. Notice that both $p_\mathsmaller{E_f}$ and $p_\mathsmaller{\varepsilon}$ are unknown and $p_\mathsmaller{\varepsilon}$ is not directly accessible through observations due to the unknown $f^\star$. However, the following theorem reminds us that $f_\mathcal{H}$ defined above may still be approached empirically.

\begin{theorem}\label{population_risk}
Let $f_\mathcal{H}$ be defined in \eqref{density_based_strategy}. Then we have the following relation
	\begin{align*}
	f_\mathcal{H}=\arg\max_{f\in\mathcal{H}}\mathbb{E}p_\mathsmaller{\varepsilon|X}(Y-f(X)),
	\end{align*}
	where the expectation is taken jointly with respect to $X$ and $Y$.
\end{theorem}
\begin{proof}
	To prove the statement, we first recall that  $f_\mathcal{H}=\arg\min_{f\in\mathcal{H}}{\sf{dist}}(p_\mathsmaller{E_f},p_\mathsmaller{\varepsilon})$, where
	\begin{align*}
	{\sf{dist}}(p_\mathsmaller{E_f},p_\mathsmaller{\varepsilon})&=\int_\mathcal{X}\int_{-\infty}^{+\infty}(p_\mathsmaller{E_f|X=x}(t)-p_\mathsmaller{\varepsilon|X=x}(t))^2\mathrm{d}t\mathrm{d}\rho_X(x)\\
	&=\int_\mathcal{X}\left[\int_{-\infty}^{+\infty}(p_\mathsmaller{E_f|X=x}(t))^2\mathrm{d}t-2\int_{-\infty}^{+\infty}p_\mathsmaller{E_f|X=x}(t)p_\mathsmaller{\varepsilon|X=x}(t)\mathrm{d}t+\int_{-\infty}^{+\infty}(p_\mathsmaller{\varepsilon|X=x}(t))^2\mathrm{d}t\right]\mathrm{d}\rho_X(x). 
	\end{align*}
	Note that the third term 
	\begin{align*}
	\int_\mathcal{X}\int_{-\infty}^{+\infty}(p_\mathsmaller{\varepsilon|X=x}(t))^2\mathrm{d}t\mathrm{d}\rho_X(x)
	\end{align*}
	is independent of $f$. Moreover, regarding the first term, we have the following relations
	\begin{align*}
	&\int_\mathcal{X}\int_{-\infty}^{+\infty}(p_\mathsmaller{E_f|X=x}(t))^2\mathrm{d}t\mathrm{d}\rho_X(x)\\
	=&\int_\mathcal{X}\int_{-\infty}^{+\infty}(p_\mathsmaller{\varepsilon|X=x}(t+f(x)-f^\star(x)))^2\mathrm{d}t\mathrm{d}\rho_X(x)\\
	=&\int_\mathcal{X}\int_{-\infty}^{+\infty}(p_\mathsmaller{\varepsilon|X=x}(t))^2\mathrm{d}t\mathrm{d}\rho_X(x),
	\end{align*}
	which imply that it is also independent of $f$. On the other hand, we have 
	\begin{align*}
	f_\mathcal{H}&=\arg\max_{f\in\mathcal{H}}\int_{\mathcal{X}}\int_{-\infty}^{+\infty}p_\mathsmaller{E_f|X=x}(t)p_\mathsmaller{\varepsilon|X=x}(t)\mathrm{d}t\mathrm{d}\rho_{X}(x)\\
	&=\arg\max_{f\in\mathcal{H}}\mathbb{E}p_\mathsmaller{\varepsilon|X}(Y-f(X)),
	\end{align*}
	where the above expectation operation is taken jointly with respect to $X$ and $Y$. This completes the proof of Theorem \ref{population_risk}. 
\end{proof}

Theorem \ref{population_risk} holds because the data-generating model \eqref{data_generating} defines a location family. Analogously, one can also prove that  $f_\mathcal{H}=\arg\max_{f\in\mathcal{H}}\mathbb{E}p_\mathsmaller{E_f|X}(Y-f^\star(X))$. While $f_\mathcal{H}$ is not directly accessible as mentioned above, one may use its empirical counterpart to approach it by assuming a prior distribution to the noise variable $\varepsilon$. Assuming a Gaussian prior, one then arrives at the formulation of the correntropy based regression scheme \eqref{MCCR}. Therefore, MCCR can be both interpreted from an information-theoretic learning viewpoint and a minimum distance estimation viewpoint, though the latter one requires the location model assumption. Such a minimum distance estimation interpretation of MCCR can help explain its robustness merit in learning problems as the robustness of minimum distance estimators has been extensively studied; see e.g., \cite{donoho1988automatic,basu2011statistical}.

\section{MCCR: A Unified Approach to Learning Location Functions}\label{sec::Unified}
In this section, we show that correntropy based regression provides us a unified approach to learning the three canonical location functions, namely, the conditional mean, median, and mode functions, which further explains its powerfulness and robustness merits in learning.

\subsection{Learning with MCCR for Mean Regression}\label{subsec::mean}
Assuming $\mathbb{E}(\varepsilon|X)=0$, we first show that MCCR can learn the conditional mean function under the following conditional $(1+\epsilon)$-moment assumption and capacity assumption.  

\begin{assumption}\label{conditional_moment}
There exist some constants $M>0$ and $\epsilon>0$ such that 
\begin{align*}
\mathbb{E}(|Y|^{1+\epsilon}|X=x)\leq M, \quad \forall x\in\mathcal{X}. 
\end{align*}
\end{assumption}

\begin{assumption}\label{complexity_assumption} 
	There exist positive constants $q$ and $c$ such that
	$$\log\mathcal{N}(\mathcal{H},\eta)\leq c \eta^{-q},\,\, \forall\,\,
	\eta>0,$$
where the covering number $\mathcal{N}(\mathcal{H},\eta)$ is defined as the minimal $k\in\mathbb{N}$ such that there exist $k$ disks in $\mathcal{H}$ with radius $\eta$ covering $\mathcal{H}$.
\end{assumption}

The above capacity condition is typical in learning theory; see e.g., \cite{cucker2007learning,steinwart2008support}. And the conditional $(1+\epsilon)$-moment restriction in Assumption \ref{conditional_moment} is a weak one as it admits the case where light-tailed noise is absent and even the case where the noise possesses infinite conditional variance.  

\begin{theorem}\label{thm::prediction_error}
	Suppose that Assumptions \ref{conditional_moment} and \ref{complexity_assumption} hold and $f^\star\in\mathcal{H}$. Let $f_{\mathbf{z},\sigma}$ be produced by \eqref{MCCR} with $\sigma>1$. For any $0<\delta<1$, with probability at least $1-\delta$, it holds that 
	\begin{align*}
	\|f_{\mathbf{z},\sigma}-f^\star\|_{2,\rho}^2\lesssim \log(2/\delta)\left(\frac{1}{\sigma^{\min\{\epsilon,2\}}}+\frac{\sigma}{n^{1/(q+1)}}\right),
	\end{align*}
	where $\|\cdot\|_{2,\rho}^2$ denotes the $L_{\rho_\mathcal{X}}^2$ norm and the sign $\lesssim$ denotes that the underlying inequality holds up to an absolute constant factor. 
\end{theorem}
 
Theorem \ref{thm::prediction_error} can be proved analogously as Theorem 2 in \cite{feng2020learning}. A sketch of its proof is provided in the appendix. Under the assumptions of Theorem \ref{thm::prediction_error}, if we set
$\sigma = n^{\Theta_\epsilon}$
where 
\begin{align*}
\Theta_\epsilon=
\begin{cases}
\frac{1}{(q+1)(\epsilon+1)},\quad\hbox{if}\quad 0<\epsilon\leq 2,\\
\frac{1}{3(q+1)},\quad\quad\,\,\,\,\hbox{if}\quad \epsilon >2,
\end{cases}
\end{align*}
then for any $0<\delta<1$, with probability at least $1-\delta$, it holds that 
\begin{align*}
\|f_{\mathbf{z},\sigma}-f^\star\|_{2,\rho}^2\lesssim \log(2/\delta)n^{-\min\{\epsilon,2\} \Theta_\epsilon}.   
\end{align*}
Therefore, with diverging $\sigma$ values, $f_{\mathbf{z},\sigma}$ approaches the conditional mean function $f^\star$. Not surprisingly, the established convergence rates depend on the capacity of $\mathcal{H}$ and the order of the moment condition in terms of the two indices $q$ and $\epsilon$. Moreover, for the case $0<\epsilon<1$ when the noise $\varepsilon$ possesses infinite conditional variance, exponential type convergence rates can still be obtained. Comparing with the results in \cite{feng2020learning}, with the conditional $(1+\epsilon)$-moment assumption, improved convergence rates are established. It is interesting to note that when $\epsilon\geq 2$, imposing higher-order conditional moment assumptions may not help improve the convergence rates of $\|f_{\mathbf{z},\sigma}-f^\star\|_{2,\rho}^2$. This phenomenon, also observed in \cite{feng2020learning}, is termed as the saturation effect in mean regression, which is caused by the introduction of the parameter $\sigma$, and is hence the cost of robustness.

\subsection{Learning with MCCR for Modal Regression}\label{subsec::modal}
MCCR can be also utilized to perform modal regression, that is, learning the conditional mode function defined in \cite{collomb1987note} as 
\begin{align}\label{condition_mode}
f_\mathsmaller{\sf{MO}}(x):=\arg\max_{t\in\mathbb{R}}p_{Y|X=x}(t),\quad x\in\mathcal{X}.
\end{align}  
In the data generating model \eqref{data_generating}, if we assume that $p_\mathsmaller{\varepsilon|X}$ admits a unique global mode for any realization of $X$, then $f_\mathsmaller{\sf{MO}}$ in \eqref{condition_mode} is well defined and is exactly $f^\star$. Recalling the results in Theorem \ref{population_risk}, we have $f_\mathsmaller{\sf{MO}}=\arg\max_{f\in\mathcal{M}}\mathbb{E}p_\mathsmaller{\varepsilon|X}(Y-f(X))$. As mentioned previously, assuming that the noise variable $\varepsilon$ is Gaussian and approximating $\mathbb{E}p_\mathsmaller{\varepsilon|X}(Y-f(X))$ by using its empirical counterpart, one can arrive at MCCR. However, directly imposing such a noise assumption seems to be a brute-force approach to learning the conditional mode. The following theorem established in \cite{feng2017statistical} provides an alternate formulation for characterizing the conditional mode function $f_\mathsmaller{\sf{MO}}$ and makes such a learning problem practically implementable. For the sake of completeness, we also provide its proof here. 
\begin{theorem}\label{density_change_variable} 
	Let $f: \mathcal{X}\rightarrow \mathbb{R}$ be any measurable function and $E_f=Y-f(X)$. Then, we have
	\begin{align*}
	f_\mathsmaller{\sf{MO}}=\arg\max_{f\in\mathcal{M}}p_\mathsmaller{E_f}(0). 
	\end{align*}
\end{theorem}
\begin{proof} 
	From the model assumption that $\varepsilon = Y - f^\star(X)$, we have
	\begin{align*}
	\varepsilon = E_f + f(X) - f^\star(X).
	\end{align*}
	As a result, the density function of the residual variable $E_f$, denoted by $p_\mathsmaller{E_f}$, can be expressed as
	\begin{align*}
	\int_\mathcal{X}p_{\varepsilon|X=x}(\cdot + f(x) - f^\star(x))\mathrm{d}\rho_{\mathsmaller{\mathcal{X}}}(x).
	\end{align*}
	Moreover, we know that  
	\begin{align*}
	p_\mathsmaller{E_f}(0)&=\int_{\mathcal{X}}p_{\varepsilon|X=x}( f(x) - f^\star(x)) \mathrm{d}\rho_{\mathsmaller{\mathcal{X}}}(x) = \int_{\mathcal{X}}p_{\mathsmaller{Y|X=x}}(f(x))\mathrm{d}\rho_{\mathsmaller{\mathcal{X}}}(x).
	\end{align*}
	This completes the proof of Theorem \ref{density_change_variable}.
\end{proof}

As a consequence of Theorem \ref{density_change_variable}, one can approach the conditional mode function through the maximization of the kernel density estimator of $p_{\mathsmaller{E_f}}$ at the point $0$, i.e., 
\begin{align}\label{estimator_densityestimator} 
f_{\mathbf{z},\sigma}=\arg\max_{f\in\mathcal{H}}\frac{1}{n\sigma}\sum_{i=1}^n\exp\left(-\frac{(y_i-f(x_i))^2}{\sigma^2}\right).
\end{align}
Note that the estimator produced in \eqref{estimator_densityestimator} is essentially the same as the MCCR estimator \eqref{MCCR}. In the statistics literature, it has been well understood that the consistency of this density estimator can be guaranteed under mild conditions, e.g., $\sigma \rightarrow 0$ and $n\sigma\rightarrow +\infty$. However, the convergence of $f_{\mathbf{z},\sigma}$ to $f_\mathsmaller{\sf{MO}}$ cannot be readily obtained from the convergence of $\mathbb{E}p_\mathsmaller{\varepsilon|X}(Y-f_{\mathbf{z},\sigma}(X))$ to $\mathbb{E}p_\mathsmaller{\varepsilon|X}(Y-f_\mathsmaller{\sf{MO}}(X))$ due to the nonconvexity of the learning scheme and so calls for some special attention. In a recent study, some efforts in this regard are made in \cite{feng2017statistical} by imposing certain assumptions on the noise variable $\varepsilon$. Exponential-type convergence rates of $f_{\mathbf{z},\sigma}$ are established there when $\sigma:=\sigma(n)\rightarrow 0$, which theoretically justifies the learnability of $f_{\mathbf{z},\sigma}$ towards the conditional mode function. It should be remarked that here the nonconvexity may only matter in learning theory analysis of $f_{\mathbf{z},\sigma}$ when deriving its convergence rates. The story may be different when assessing it from an optimization viewpoint \cite{syed2014optimization}.

\subsection{Learning with MCCR for Median Regression}\label{subsec::median}
We now provide some perspectives on learning with MCCR for median regression. Under the regression model \eqref{data_generating} and the zero-median assumption ${\sf{median}}(\varepsilon|X)=0$, Theorem \ref{population_risk} tells us that as the population version of $f_{\mathbf{z},\sigma}$, $f_\mathsmaller{\mathcal{H}}$ maximizes $\mathbb{E}p_{\varepsilon|X}(Y-f(X))$ over $\mathcal{H}$. However, this neither implies the convergence of $f_{\mathbf{z},\sigma}$ to the conditional median  $f^\star$ nor indicates the convergence of $f_\mathcal{H}$ to $f^\star$. 

To see that $f_{\mathbf{z},\sigma}$ can serve as a median regression estimator, we consider a special case when the noise variable $\varepsilon$ is independent of the input variable $X$ and is symmetric stable, i.e., its characteristic function $\phi_\varepsilon$ admits the form
$\phi_\varepsilon(t)=e^{-\gamma |t|^\alpha}$,    
where $\gamma>0$ is a constant, and $0<\alpha\leq 2$ is the characteristic exponent. It is well known that, the normal distribution is stable with $\alpha=2$ and the Cauchy distribution is stable with $\alpha=1$. When $\alpha<2$, absolute moments of order less than $\alpha$ exist while those of order greater than or equal to $\alpha$ do not. Therefore, under the zero median assumption, $f^\star$ is, in fact, the conditional median function as the conditional mean function may not even be defined. According to \cite{feng2018learning}, in this case, MCCR can learn the conditional median function $f^\star$ well in the sense that   $\mathbb{E}p_\mathsmaller{\varepsilon|X}(Y-f_{\mathbf{z},\sigma}(X))\rightarrow \mathbb{E}p_\mathsmaller{\varepsilon|X}(Y-f^\star(X))$ implies $f_{\mathbf{z},\sigma}\rightarrow f^\star$ with a proper fixed  $\sigma$. Moreover, fast exponential-type convergence rates can be established. However, whether MCCR can learn the conditional median function $f^\star$ under more general conditions is still yet to be explored.

\section{Numerical Validations}\label{sec::experiments}
In this section, we conduct numerical simulations on synthetic data to validate our theoretical finding that MCCR provides a unified approach to learning location functions. 

To this end, as in \cite{feng2017statistical}, we consider the regression model
$y=f^\star(x)+\varepsilon$ where $x\sim U(0,1)$, $f^\star(x)=2\sin(\pi x)$, and the noise obeys the following two different distributions:
\begin{itemize}
    \item Case I: $\varepsilon=(1+2x)\kappa$ with  $\kappa \sim 0.5 N(-1,2.5^2)+0.5N(1,0.5^2)$;
    \item Case II: $\varepsilon \sim \hbox{Cauchy}(0,0.5)$.
\end{itemize}
For Case I, with simple computations, we know that the conditional mean function is 
$\mathbb{E}(Y|X)=2\sin(\pi x)$, and the conditional mode function is approximately ${\sf{mode}(Y|X)}=2\sin(\pi x)+1+2x$. For Case II, it is obvious that ${\sf{median}}(Y|X)=2\sin(\pi x)$. 

In our experiments, $200$ observations are drawn from the above data-generating model for training and the size of the test set is also $200$. The hypothesis space $\mathcal{H}$ is chosen as a subset of the Gaussian reproducing kernel Hilbert space by using Tikhonov regularization. The bandwidth of the Gaussian kernel and the regularization parameter are selected through five-fold cross-validation under the least absolute deviation criterion. Three experiments are conducted, respectively, in order to show that $f_{\mathbf{z},\sigma}$ can approach the three different location functions with different $\sigma$ values. The learned functions from the three experiments are plotted in Figs.\,1-3. For each experiment, the $\sigma$ value is set to be fixed and is specified in the captions of the three figures. Clearly, from the experiments, we see that with different choices of $\sigma$ values, $f_{\mathbf{z},\sigma}$  can  indeed approach the three location functions, which consequently demonstrates our theoretical finding empirically.

\begin{figure}[!h]
\tikzset{trim left=0.0 cm}
       \setlength\figureheight{4cm}
       \setlength\figurewidth{4cm}
      \begin{minipage}[b]{0.2\textwidth}
                   \input{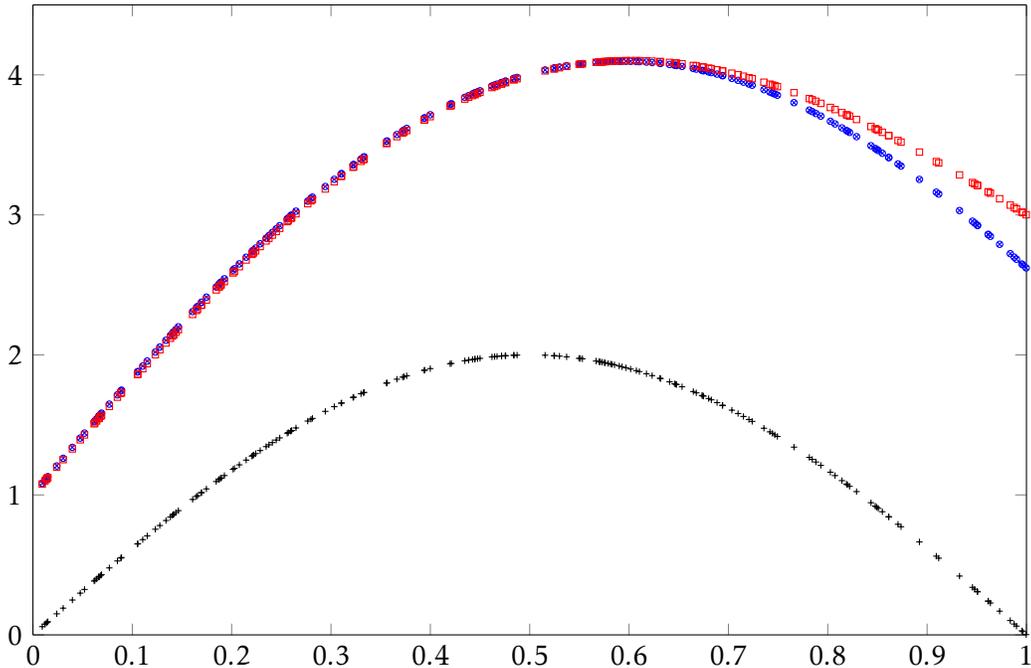}
                 \end{minipage} 
  \caption{Experimental results of Case I: the  red curve with square marks denotes the conditional mode function. The  black curve with plus marks gives the conditional mean function. The  blue curve with $\otimes$ marks represents the learned estimator $f_{\mathbf{z},\sigma}$ with $\sigma=0.05$.}\label{toy_mode}
\end{figure}

\begin{figure}[!h]
\tikzset{trim left=0.0 cm}
       \setlength\figureheight{4cm}
       \setlength\figurewidth{4cm}
               \begin{minipage}[b]{0.2\textwidth}
                \input{mean_estimator.tikz}
           \end{minipage}%
  \caption{Experimental results of Case I: the red curve with square marks denotes the conditional mode function. The black curve with plus marks gives the conditional mean function. The blue curve with $\otimes$ marks represents the learned estimator $f_{\mathbf{z},\sigma}$ with $\sigma=10$.}\label{toy_mean}
\end{figure} 

\begin{figure}[!h]
\tikzset{trim left=0.0 cm}
       \setlength\figureheight{4cm}
       \setlength\figurewidth{4cm}
               \begin{minipage}[b]{0.2\textwidth}
                \input{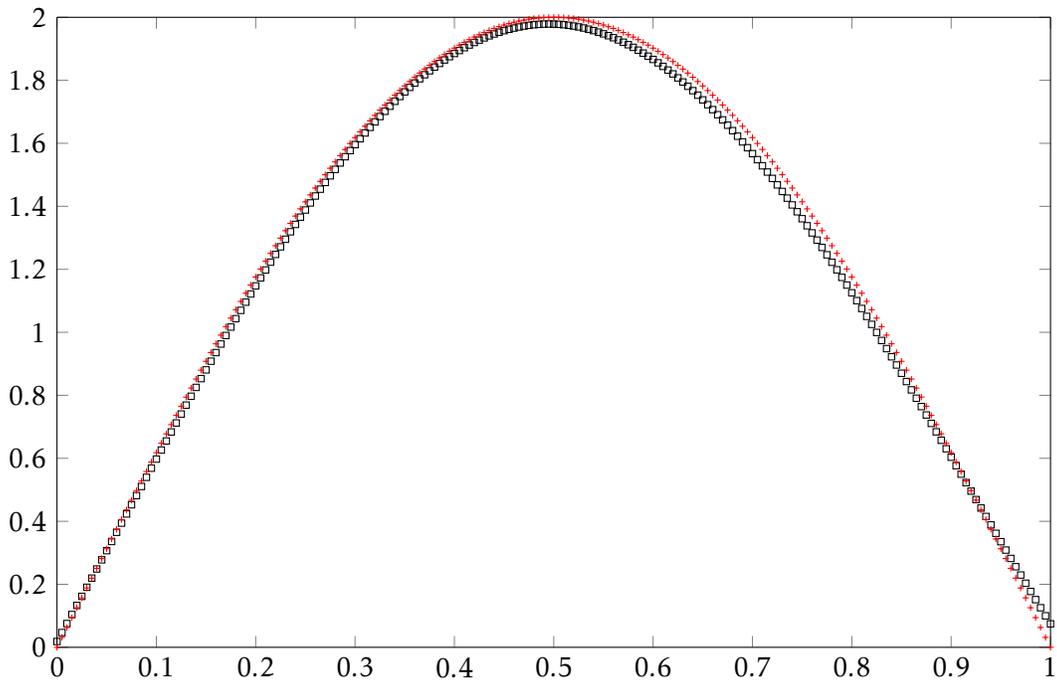}
           \end{minipage}%
  \caption{Experimental results of Case II: the red curve with square marks denotes the conditional median function. The black curve with plus marks represents the learned estimator $f_{\mathbf{z},\sigma}$ with $\sigma=0.01$.}\label{toy_median}
\end{figure} 

\section{Conclusion and Future Work}\label{sec::conclusion}
In this paper, we studied the correntropy based regression by drawing some novel insights into it. We first concluded that the resulting regression estimator can be viewed as a minimum distance estimator, which helps understand its robustness property. Moreover, this finding indicates its practical applicability as it simultaneously enjoys the nice properties of both information-theoretic learning and minimum distance estimation. We then showed that the regression estimator can work effectively in learning the unknown truth function as it is capable of learning different location functions under certain noise restrictions by tuning the scale parameter adaptively. When learning the mean regression function, the established exponential type convergence rates under weak conditional moment assumptions remind us of the existence of the saturation effect caused by some inherent bias. These insights widen our understanding of the regression scheme, help cement the theoretic correntropy framework, and also enable us to investigate learning schemes induced by general bounded nonconvex loss functions.

Yet, there are still several problems that need to be addressed in order to paint a complete picture of the regression paradigm. Here we exemplify several of the problems. First, the saturation effect reported in this study limits the learnability of the resulting regression estimator. This is because even in the presence of light-tailed noise, e.g., skewed Gaussian, the established convergence rates in mean regression can only be up to $\mathcal{O}(n^{-2/3})$ which are not even comparable with those of least squares regression estimators in the same situation. Therefore, further efforts should be made to debias when implementing correntropy based mean regression. Second, we discussed above only a specific case when correntropy based regression regresses towards the conditional median function by requiring that the noise is symmetric stable. It is still unknown whether one could further relax such a stringent restriction on the noise or what kind of $\sigma$ values one should choose. Third, we only investigated here the problem of learning location functions through MCCR. It would be also interesting to investigate the problem of learning generalized location functions using similar approaches. In addition, in the present study, we only consider the hypothesis space in which all functions are uniformly bounded. In practice, the hypothesis space is typically automatically chosen by a penalized ERM scheme where functions are generally no longer uniformly bounded. It is still unclear in this case how to derive exponential-type convergence rates of the correntropy based regression estimators without imposing light-tailed noise assumptions on the noise. The above-exemplified research problems illustrate our future work on this topic.

\section*{Appendix} 
In this appendix section, we provide a sketch of the proof of Theorem \ref{thm::prediction_error}, which is accomplished by using similar arguments as in the proof of Theorem 2 in \cite{feng2020learning}. The key difference is that, in the present study, with the conditional $(1+\epsilon)$-moment restriction stated in Assumption \ref{conditional_moment}, one can obtain refined variance estimates for $\xi$ defined below, which lead to improved convergence rates.   

Before proving the theorem, we first introduce some notation. For any measurable function $f:\mathcal{X}\rightarrow \mathbb{R}$, we denote 
\begin{align*}
\mathcal{R}^\sigma(f)=\mathbb{E}\ell_\sigma(Y-f(X))
\end{align*}
as its generalization error and denote its empirical generalization error as
\begin{align*}
\mathcal{R}^\sigma_\mathbf{z}(f)=\frac{1}{n}\sum_{i=1}^n\ell_\sigma(y_i-f(x_i)).
\end{align*}
We further denote $f_{\mathcal{H},\sigma}$ as the population version of $f_{\mathbf{z},\sigma}$ in $\mathcal{H}$, that is,  
	\begin{align*}
	f_{\mathcal{H},\sigma}:=\arg\min_{f\in\mathcal{H}}\mathcal{R}^\sigma(f).
	\end{align*}
	We also denote $f_\mathcal{H}$ as the ``best" function in $\mathcal{H}$ when approximating $f^\star$ in the following sense
	\begin{align*}
	 f_\mathcal{H}=\arg\min_{f\in\mathcal{H}}\|f-f^\star\|_{2,\rho}^2.
	 \end{align*}
Under Assumption \ref{conditional_moment}, Theorem 1 in \cite{feng2020learning} tells us that for any measurable function $f:\mathcal{X}\rightarrow \mathbb{R}$ with $\|f\|_\infty\leq M$ and $\sigma>1$, it holds that	
	\begin{align}\label{comparison_inequality} 
	\Big|\left[\mathcal{R}^\sigma(f)-\mathcal{R}^\sigma(f^\star)\right]- \|f-f^\star\|_{2,\rho}^2 \Big|\leq\cfrac{c_\mathsmaller{\mathcal{H},\epsilon}}{\sigma^{\theta_\epsilon}},
	\end{align}
	where, for any fixed $\epsilon$, the constant $\theta_\epsilon$ is given by
	$\theta_\epsilon=\min\{\epsilon,2\}$,
	and  
	$c_\mathsmaller{\mathcal{H},\epsilon}$ is an absolute constant independent of $f$ or $\sigma$.

    To prove the theorem, for any $f\in\mathcal{H}$, we denote $\xi(x,y)$ as the following random variable 
	\begin{align*}
	\xi(x,y)=\ell_\sigma(y-f(x))-\ell_\sigma(y-f^\star(x)),\, (x,y)\in\mathcal{X}\times\mathcal{Y}.
	\end{align*}
    We can bound the variance of the random variable $\xi$ by considering two different cases of the $\epsilon$ values. When $\epsilon\geq  1$, we have  
	\begin{align*}
	\hbox{var}(\xi)\leq\mathbb{E}\xi^2 &\leq   \mathbb{E}\left(\ell_\sigma(y-f(x))-\ell_\sigma(y-f^\star(x))\right)^2\\
	&\leq \mathbb{E}\left((y-f(x))^2-(y-f^\star(x))^2\right)^2
	\\
	&\leq c_1  \|f-f^\star\|_{2,\rho}^{2},
	\end{align*}
	where $c_1 =18M^2$, the third inequality is obtained by applying the mean value theorem. When $0<\epsilon<1$, the variance of $\xi$ can be bounded as follows 
	\begin{align*}
	\hbox{var}(\xi)\leq \mathbb{E}\xi^2 &\leq   \mathbb{E}\left(\ell_\sigma(y-f(x))-
	\ell_\sigma(y-f^\star(x))\right)^2\\
	&\leq \sigma^{1-\varepsilon} \|f-f^\star\|_\infty ^{1-\epsilon}\mathbb{E}\left|\ell_\sigma(y-f^\star(x))-
	\ell_\sigma(y-f(x))\right|^{1+\epsilon} \\
	&\leq \sigma^{1-\epsilon}((3M)^{1+\epsilon}+3^\epsilon  \mathbb{E}|Y|^{1+\epsilon})\|f-f^\star\|^2_\infty \leq c_2  \sigma^{1-\epsilon},
	\end{align*}
	where $c_2 =2M^2 ((3M)^{1+\epsilon}+3^\epsilon  \mathbb{E}|Y|^{1+\epsilon})$, and the third and the fourth inequalities are again obtained by applying the mean value theorem. Then, using the similar arguments as in the proof of Theorem 2 in \cite{feng2020learning}, one can accomplish the proof through the following key steps. 
    
    First, under Assumption \ref{conditional_moment} and $\sigma>1$, for any $\gamma\geq c_\mathsmaller{\mathcal{H},\epsilon}\sigma^{-\theta_\epsilon}$, with probability at most $\mathcal{N}\left(\mathcal{H}, \gamma \sigma^{-1}\right) e^{-\frac{n\gamma}{c_3\sigma}}$, it holds that 
	\begin{align*}
	\sup_{f\in\mathcal{H}}\left\{
	\frac{\big|[\mathcal{R}^{\sigma}(f)-\mathcal{R}^{\sigma}(f^\star)]-
		[\mathcal{R}^\sigma_{\mathbf{z}}(f)-\mathcal{R}^\sigma_{\mathbf{z}}(f^\star)]\big|}
	{\sqrt{\mathcal{R}^{\sigma}(f)-\mathcal{R}^{\sigma}(^\star)+2\gamma}}\right\}
	>4\sqrt{\gamma},
	\end{align*}
	where $c_3$ is a positive constant independent of $\sigma$. This probability ratio inequality is established by applying the one-sided Bernstein inequality and utilizing the compactness as well as the complexity assumption of the hypothesis space $\mathcal{H}$. 
	
Second, denoting  
\begin{align*}
	\gamma_0= 
	\frac{1}{\sigma^{\theta_\epsilon}}+\log\left(\frac{2}{\delta}\right)\frac{\sigma}{n^{1/(q+1)}},
\end{align*} 
then one can prove that for any $0<\delta<1$, with probability at least $1-\delta/2$, it holds that
	\begin{align*}
	[\mathcal{R}^{\sigma}(f_{\mathbf{z},\sigma})-\mathcal{R}^{\sigma}(f^\star)]-
	[\mathcal{R}^\sigma_{\mathbf{z}}(f_{\mathbf{z},\sigma})-\mathcal{R}^\sigma_{\mathbf{z}}(f^\star)]- \frac{1}{2}[\mathcal{R}^{\sigma}(f_{\mathbf{z},\sigma})-\mathcal{R}^{\sigma}(f^\star)]\lesssim \gamma_0,
	\end{align*}
and that
	\begin{align*}
	[\mathcal{R}^\sigma_{\mathbf{z}}(f_{\mathcal{H},\sigma})-\mathcal{R}^\sigma_{\mathbf{z}}(f^\star)]-[\mathcal{R}^\sigma(f_{\mathcal{H},\sigma})-\mathcal{R}^\sigma(f^\star)]- \frac{1}{2}\|f_\mathcal{H}-f^\star\|_{2,\rho}^2\lesssim \gamma_0.
	\end{align*}

Third, combining the above two estimates, with simple computations, it can be shown that for any $0<\delta<1$, with probability at least $1-\delta$, one has
	\begin{align*}
	\|f_{\mathbf{z},\sigma}-f^\star\|_{2,\rho}^2\lesssim \|f_\mathcal{H}-f^\star\|_{2,\rho}^2+\log(2/\delta)\left(\frac{1}{\sigma^{\theta_\epsilon}}+\frac{\sigma}{n^{1/(q+1)}}\right).
	\end{align*}
Recalling that $f^\star\in\mathcal{H}$, we have $\|f_\mathcal{H}-f^\star\|_{2,\rho}^2=0$ and thus arrive at the desired error bound. This gives a sketch of the proof of Theorem \ref{thm::prediction_error}.

\section*{Acknowledgement}
The author would like to thank the reviewers and Dr. Qiang Wu for insightful comments that improved the quality of this paper. The work of the author was partially supported by the Simons Foundation Collaboration Grant \#572064 and the Ralph E. Powe Junior Faculty Enhancement Award by Oak Ridge Associated Universities.

\bibliographystyle{plain}
\bibliography{library}

\end{document}